\makeatletter\def\input@path{{tables/}{sections/}}\makeatother
\documentclass{article} %

\usepackage{iclr2027_conference,times}

\usepackage{float}
\usepackage[hidelinks]{hyperref}
\usepackage{url}
\usepackage{booktabs}
\usepackage{amsmath,amssymb}
\usepackage{graphicx}
\graphicspath{{figures/}}
\usepackage{microtype}
\usepackage{tabularx}
\usepackage{tikz}
\usetikzlibrary{positioning,calc,arrows}
\usetikzlibrary{arrows.meta}
\definecolor{trBlue}{HTML}{0072B2}
\definecolor{trGreen}{HTML}{00674C}
\definecolor{trVerm}{HTML}{D55E00}
\definecolor{trGrey}{HTML}{8C8C8C}
\definecolor{trInk}{HTML}{262626}

\usepackage{capt-of}
\usepackage{amsmath}
\usepackage{amssymb}
\usepackage{amsthm}
\title{Measure First, Then Allocate:\\Calibrated Rollout Budgets\\for MCP Agents in Reinforcement Learning}

\iclrfinalcopy
\author{Hieu Tran$^{1}$\quad Zonghai Yao$^{2}$\quad Sunjae Kwon$^{1}$\quad Hong Yu$^{1,2}$\\
\textnormal{$^{1}$University of Massachusetts Amherst\quad $^{2}$University of Massachusetts Lowell}\\
\textnormal{\texttt{hieutran@umass.edu}}}

\DeclareRobustCommand{\methodname}{\textsc{Brace}}
\usepackage{xcolor}
\usepackage{fancyhdr}
\newif\ifpendingused\pendingusedfalse

\makeatletter
\renewcommand\floatc@ruled[2]{{\@fs@cfont #1:} #2\par}
\makeatother
\floatstyle{ruled}
\newfloat{algorithm}{tb}{loa}
\floatname{algorithm}{Algorithm}

\newtheorem{proposition}{Proposition}

\begin{document}

\maketitle
\lhead{Preprint}
\suppressfloats[t]

\providecommand{\num}[1]{#1}

\begin{abstract}
Training tool-using language models requires costly interaction, yet identical
rewards within a rollout group provide no group-relative learning signal. We
use reward disagreement to guide harness evaluation and the allocation of
rollouts for policy training in Model Context Protocol (MCP) environments. We first introduce \textbf{Discord},
which tests harness edits using exact paired inference and decision-preserving
early stopping: only disagreeing paired outcomes favor one harness over another.
This insight motivates prospective allocation in
\textbf{\methodname{}} (\textbf{B}ank-calibrated \textbf{R}ollout \textbf{A}llocation under \textbf{C}orrelated \textbf{E}pisodes).
\methodname{} combines historical task evidence, uncertainty in task success
rates, and a correlated reward model to forecast mixed-reward groups before
allocating generation budget. Controlled harness experiments favor
\textbf{Discord}, while live MCP experiments expose a retention--cost trade-off.
At three Qwen3-VL model scales, \textbf{\methodname{}} exceeds the strongest plotted
baseline peaks with fewer new training rollouts; historical-bank collection is
counted separately. It also leads the compared methods on BFCL and NESTFUL without
further training. Ablations probe the bank, estimator, and allocation score;
prospective evaluation checks forecast calibration. These
results support forecasting reward contrast before generation.
\footnote{Code can be accessed at: \url{https://anonymous.4open.science/r/brace-38DC}}
\end{abstract}

\providecommand{\trSq}[4]{%
  \filldraw[draw=#4, fill=#3, line width=0.4pt, rounded corners=0.4pt]
    ({#1-0.11},{#2-0.11}) rectangle ({#1+0.11},{#2+0.11});}
\providecommand{\trSqS}[4]{%
  \filldraw[draw=#4, fill=#3, line width=0.4pt, rounded corners=0.35pt]
    ({#1-0.09},{#2-0.09}) rectangle ({#1+0.09},{#2+0.09});}
\providecommand{\trPair}[4]{%
  \trSqS{#1}{5.37}{#3}{#2}\trSqS{#1}{5.11}{#4}{#2}}
\providecommand{\trGhost}[1]{%
  \foreach \gy in {5.37,5.11}{%
    \draw[draw=trGrey!60, line width=0.35pt, rounded corners=0.35pt, densely dotted]
      ({#1-0.09},{\gy-0.09}) rectangle ({#1+0.09},{\gy+0.09});}}
\providecommand{\trChip}[4]{%
  \filldraw[draw=#4, fill=#3, line width=0.35pt, rounded corners=0.3pt]
    ({#1-0.08},{#2-0.105}) rectangle ({#1+0.08},{#2+0.105});}
\providecommand{\trNil}[3]{%
  \draw[#3, line width=0.8pt, line cap=round] ({#1-0.09},{#2}) -- ({#1+0.09},{#2});}
\providecommand{\trBadOK}[3]{%
  \filldraw[draw=#3, fill=white, line width=0.55pt] ({#1},{#2}) circle (0.13);
  \draw[#3, line width=0.7pt, line cap=round, line join=round]
    ({#1-0.059},{#2+0.005}) -- ({#1-0.016},{#2-0.047}) -- ({#1+0.068},{#2+0.059});}
\providecommand{\trBadX}[3]{%
  \filldraw[draw=#3, fill=#3, line width=0.3pt] ({#1},{#2}) circle (0.13);
  \draw[white, line width=0.85pt, line cap=round]
    ({#1-0.059},{#2-0.059}) -- ({#1+0.059},{#2+0.059})
    ({#1-0.059},{#2+0.059}) -- ({#1+0.059},{#2-0.059});}
\providecommand{\trNum}[3]{%
  \filldraw[draw=trBlue, fill=trBlue, line width=0.3pt] ({#1},{#2}) circle (0.135);
  \node[font=\scriptsize\bfseries, text=white, inner sep=0pt] at ({#1},{#2}) {#3};}
\providecommand{\trCyl}[3]{%
  \path[draw=#3!70, fill=#3!12, line width=0.45pt]
    ({#1-0.125},{#2+0.115}) -- ({#1-0.125},{#2-0.115})
    arc[start angle=180, end angle=360, x radius=0.125, y radius=0.048]
    -- ({#1+0.125},{#2+0.115})
    arc[start angle=0, end angle=180, x radius=0.125, y radius=0.048] -- cycle;
  \draw[draw=#3!70, fill=#3!28, line width=0.45pt]
    ({#1},{#2+0.115}) ellipse (0.125 and 0.048);
  \draw[#3!45, line width=0.3pt]
    ({#1-0.125},{#2+0.005}) arc[start angle=180, end angle=360, x radius=0.125, y radius=0.048];}
\providecommand{\trGauge}[3]{%
  \path[draw=#3!70, fill=#3!10, line width=0.45pt]
    ({#1-0.135},{#2-0.055}) arc[start angle=180, end angle=0, radius=0.135] -- cycle;
  \foreach \a in {155,90,25}{\draw[#3!70, line width=0.35pt]
    ({#1},{#2-0.055}) ++(\a:0.098) -- ++(\a:0.037);}
  \draw[#3!90, line width=0.55pt, line cap=round]
    ({#1},{#2-0.055}) -- ++(52:0.115);
  \fill[#3!90] ({#1},{#2-0.055}) circle (0.024);}
\providecommand{\trStack}[2]{%
  \filldraw[draw=trBlue!70, fill=trBlue!32, line width=0.32pt]
    ({#1-0.125},{#2-0.115}) rectangle ({#1+0.125},{#2-0.055});
  \filldraw[draw=trBlue!70, fill=trBlue!20, line width=0.32pt]
    ({#1-0.125},{#2-0.030}) rectangle ({#1+0.045},{#2+0.030});
  \filldraw[draw=trBlue!70, fill=trBlue!10, line width=0.32pt]
    ({#1-0.125},{#2+0.055}) rectangle ({#1-0.025},{#2+0.115});}
\providecommand{\trShield}[3]{%
  \path[draw=#3!85, fill=#3!12, line width=0.45pt]
    ({#1},{#2+0.135}) -- ({#1+0.100},{#2+0.092}) -- ({#1+0.100},{#2-0.028})
    .. controls ({#1+0.100},{#2-0.098}) and ({#1+0.050},{#2-0.122}) .. ({#1},{#2-0.150})
    .. controls ({#1-0.050},{#2-0.122}) and ({#1-0.100},{#2-0.098}) .. ({#1-0.100},{#2-0.028})
    -- ({#1-0.100},{#2+0.092}) -- cycle;
  \draw[#3!85, line width=0.6pt, line cap=round, line join=round]
    ({#1-0.046},{#2+0.010}) -- ({#1-0.012},{#2-0.030}) -- ({#1+0.056},{#2+0.050});}
\providecommand{\trClus}[3]{%
  \filldraw[draw=#3!85, fill=#3!35, line width=0.3pt, rounded corners=0.25pt]
    ({#1-0.115},{#2-0.02}) rectangle ({#1-0.015},{#2+0.08});
  \filldraw[draw=#3!85, fill=white, line width=0.3pt, rounded corners=0.25pt]
    ({#1+0.015},{#2-0.02}) rectangle ({#1+0.115},{#2+0.08});
  \filldraw[draw=#3!85, fill=#3!35, line width=0.3pt, rounded corners=0.25pt]
    ({#1-0.05},{#2-0.135}) rectangle ({#1+0.05},{#2-0.035});}

\DeclareRobustCommand{\trCap}[1]{\tikz[baseline=-0.62ex]{%
  \node[circle, fill=trBlue, text=white, inner sep=0.5pt, minimum size=1.35ex,
        font=\bfseries\tiny] {#1};}}

\begin{figure}[t]
\centering
\begin{tikzpicture}[
  font=\scriptsize,
  lane/.style     ={rounded corners=3pt, line width=0.4pt},
  slabA/.style    ={draw=trBlue!80, fill=trBlue!20, rounded corners=2pt, line width=0.4pt},
  slabtx/.style   ={rotate=90, font=\scriptsize, inner sep=0pt},
  hd/.style       ={anchor=west, font=\scriptsize, text=trInk, inner sep=0pt},
  sub/.style      ={anchor=west, font=\scriptsize, text=trGrey!75!trInk, inner sep=0pt},
  co/.style       ={anchor=west, font=\scriptsize, text=trInk!88, inner sep=0pt},
  rowlb/.style    ={anchor=east, font=\scriptsize, text=trInk!80, inner sep=0pt},
  title/.style    ={anchor=north west, font=\footnotesize\bfseries, text=trInk, inner sep=0pt},
  flowA/.style    ={-{Stealth[length=4pt,width=3.2pt]}, draw=trInk!62, line width=0.55pt},
  flowD/.style    ={-{Stealth[length=4pt,width=3.2pt]}, draw=trGreen!62, line width=0.55pt},
  bridge/.style   ={-{Stealth[length=4.6pt,width=3.6pt]}, draw=trBlue!80, line width=0.9pt},
  cut/.style      ={draw=trGreen!70, line width=0.5pt, densely dashed},
  fat/.style      ={-{Stealth[length=5.4pt,width=4.8pt]}, draw=trBlue!85, line width=1.6pt},
  loop/.style     ={-{Stealth[length=4pt,width=3.2pt]}, draw=trBlue!65, line width=0.55pt,
                    densely dotted},
  chipD/.style    ={draw=trGreen!75, fill=trGreen!10, rounded corners=2.5pt, line width=0.45pt,
                    align=center, inner sep=2pt, minimum width=1.85cm, minimum height=0.56cm,
                    font=\scriptsize, text=trInk},
  chipA/.style    ={draw=trBlue!75, fill=trBlue!10, rounded corners=2.5pt, line width=0.45pt,
                    align=center, inner sep=2pt, minimum width=1.75cm, minimum height=0.56cm,
                    font=\scriptsize}]

\useasboundingbox (0,0) rectangle (13.95,6.00);
\colorlet{laneA}{trBlue!4}
\colorlet{okA}{trBlue!80}\colorlet{edA}{trBlue!85}
\colorlet{dead}{trGrey!45}\colorlet{edD}{trGrey!85}
\colorlet{laneD}{trGreen!4}
\colorlet{okD}{trGreen!80}\colorlet{edG}{trGreen!85}
\colorlet{tie}{trGrey!45}\colorlet{edT}{trGrey!85}

\fill[laneD, lane] (0,4.76) rectangle (13.95,5.66);
\draw[trGreen!45, lane] (0,4.76) rectangle (13.95,5.66);
\node[title] at (0,6.00) {(a) \textsc{Discord}: test harness edits};

\node[rowlb] at (1.14,5.37) {incumbent};
\node[rowlb] at (1.14,5.11) {candidate};

\trPair{1.34}{edT}{tie}{tie}
\trPair{1.60}{edT}{white}{white}
\trPair{1.86}{edG}{white}{okD}
\trPair{2.12}{edT}{white}{white}
\trPair{2.38}{edT}{tie}{tie}
\trPair{2.64}{edG}{white}{okD}
\trPair{2.90}{edT}{white}{white}
\trPair{3.16}{edT}{tie}{tie}
\draw[cut] (3.32,4.92) -- (3.32,5.56);
\trGhost{3.46}\trGhost{3.72}

\draw[flowD] (3.92,5.24) -- (4.16,5.24);
\node[chipD] at (5.15,5.24) {exact test\\at $\alpha=.05$};
\draw[flowD] (6.16,5.24) -- (6.44,5.24);
\trBadOK{6.66}{5.24}{trGreen!85} \node[hd] at (6.86,5.24) {keep incumbent};

\trSqS{8.75}{5.36}{tie}{edT}\trSqS{8.75}{5.12}{tie}{edT}
\node[sub] at (8.95,5.24) {concordant = no signal};
\draw[cut] (11.15,5.06) -- (11.15,5.42);
\node[sub] at (11.33,5.24) {stop early once settled};

\node[co] at (6.95,4.60)
  {Discord's insight motivates forecasting useful rollout groups};

\fill[laneA, lane] (0,0.10) rectangle (13.95,4.12);
\draw[trBlue!45, lane] (0,0.10) rectangle (13.95,4.12);
\node[title] at (0,4.44) {(b) \methodname{}: forecast, then allocate};

\foreach \s [count=\j from 0] in {14,50,88}
  {\filldraw[draw=trBlue!55, fill=trBlue!\s, line width=0.35pt]
     (7.83+0.24*\j,4.18) rectangle (8.07+0.24*\j,4.40);}
\node[sub] at (8.71,4.29) {$P(\text{degenerate})$};
\trSq{10.50}{4.29}{okA}{edA}    \node[sub] at (10.66,4.29) {solved};
\trSq{11.57}{4.29}{white}{edA}  \node[sub] at (11.73,4.29) {unsolved};
\trSq{12.88}{4.29}{dead}{edD}   \node[sub] at (13.04,4.29) {all-equal};

\trNum{0.98}{3.82}{1}
\trCyl{1.30}{3.82}{trBlue}   \node[hd]  at (1.52,3.82) {warm bank};
                             \node[co]  at (1.52,3.58) {from earlier runs};
\trNum{3.48}{3.82}{2}
\trGauge{3.80}{3.82}{trBlue} \node[hd]  at (4.02,3.82) {calibrated forecast};
                             \node[co]  at (4.02,3.58) {models correlation};
\trNum{6.14}{3.82}{3}
\trStack{6.46}{3.82}         \node[hd]  at (6.68,3.82) {weighted task selection};
                             \node[co]  at (6.68,3.58) {chips show relative weight};
\trNum{9.56}{3.82}{4}
\trClus{9.88}{3.82}{trBlue}  \node[hd]  at (10.10,3.82) {budget lands where};
                             \node[co]  at (10.10,3.58) {signal is expected};

\draw[slabA] (3.20,0.66) rectangle (3.66,3.34);
  \node[slabtx, text=trBlue!85] at (3.43,2.00) {\bfseries forecast};
\draw[slabA] (5.30,0.66) rectangle (5.76,3.34);
  \node[slabtx, text=trBlue!85] at (5.53,2.00) {\bfseries allocate};
\draw[slabA] (8.04,0.66) rectangle (8.50,3.34);
  \node[slabtx, text=trBlue!85] at (8.27,2.00) {\bfseries generate};
\draw[flowA] (2.88,2.00) -- (3.16,2.00);
\draw[flowA] (4.98,2.00) -- (5.26,2.00);
\draw[flowA] (7.72,2.00) -- (8.00,2.00);

\foreach \y/\t in {3.10/1, 2.66/2, 2.22/3, 1.78/4, 1.34/5, 0.90/6}
  {\node[rowlb] at (0.72,\y) {task~\t};}

\foreach \c [count=\j from 0] in {okA,okA,okA,okA,okA,okA}   {\trSq{1.20+0.29*\j}{3.10}{\c}{edA}}
\foreach \c [count=\j from 0] in {okA,okA,okA,okA,white}     {\trSq{1.20+0.29*\j}{2.66}{\c}{edA}}
\foreach \c [count=\j from 0] in {okA,okA,white,okA,white}   {\trSq{1.20+0.29*\j}{2.22}{\c}{edA}}
\foreach \c [count=\j from 0] in {okA,white,okA,white,okA,white}
                                                            {\trSq{1.20+0.29*\j}{1.78}{\c}{edA}}
\foreach \c [count=\j from 0] in {okA,okA,white,okA}         {\trSq{1.20+0.29*\j}{1.34}{\c}{edA}}
\foreach \c [count=\j from 0] in {white,white,white,white,white}
                                                            {\trSq{1.20+0.29*\j}{0.90}{\c}{edA}}

\foreach \y/\s in {3.10/90, 2.66/64, 2.22/28, 1.78/8, 1.34/47}
  {\filldraw[draw=trBlue!55, fill=trBlue!\s, line width=0.4pt, rounded corners=0.7pt]
     (3.82,\y-0.105) rectangle (4.86,\y+0.105);}
\filldraw[draw=trGrey!70, fill=trGrey!35, line width=0.4pt, rounded corners=0.7pt]
  (3.82,0.795) rectangle (4.86,1.005);
\draw[trGrey!85, line width=0.7pt] (3.77,0.90) -- (4.91,0.90);

\trChip{5.96}{3.10}{trBlue!62}{edA}
\foreach \j in {0,...,2}{\trChip{5.96+0.195*\j}{2.66}{trBlue!62}{edA}}
\foreach \j in {0,...,6}{\trChip{5.96+0.195*\j}{2.22}{trBlue!62}{edA}}
\foreach \j in {0,...,8}{\trChip{5.96+0.195*\j}{1.78}{trBlue!62}{edA}}
\foreach \j in {0,...,4}{\trChip{5.96+0.195*\j}{1.34}{trBlue!62}{edA}}
\trShield{6.04}{0.90}{trBlue} \node[sub] at (6.24,0.90) {excluded task};

\trNil{8.75}{3.10}{trGrey!85}
\node[sub] at (9.03,3.10) {not selected};
\foreach \c [count=\j from 0] in {okA,okA,white,okA,white}   {\trSq{8.68+0.30*\j}{2.66}{\c}{edA}}
\foreach \c [count=\j from 0] in {okA,white,okA,okA,white}
                                                            {\trSq{8.68+0.30*\j}{2.22}{\c}{edA}}
\foreach \c [count=\j from 0] in {white,okA,okA,white,okA}
                                                            {\trSq{8.68+0.30*\j}{1.78}{\c}{edA}}
\foreach \c [count=\j from 0] in {dead,dead,dead,dead,dead}  {\trSq{8.68+0.30*\j}{1.34}{\c}{edD}}
\foreach \y in {2.66,2.22,1.78}{\trBadOK{11.44}{\y}{trBlue!85}}
\trBadX{11.44}{1.34}{trGrey!85}
\trNil{8.75}{0.90}{trGrey!85}

\draw[trBlue!35, line width=0.6pt] (11.67,1.34) -- (11.67,2.66);
\foreach \y in {2.66,2.22,1.78,1.34}
  {\draw[trBlue!45, line width=0.4pt] (11.57,\y) -- (11.67,\y);}
\draw[fat] (11.67,2.00) -- (12.05,2.00);
\node[chipA] (updA) at (12.96,2.00) {$\nabla$\,policy update};
\draw[loop] (12.96,1.72) -- (12.96,0.38) -- (0.88,0.38) -- (0.88,2.00) -- (1.07,2.00);
\node[font=\scriptsize, text=trBlue!85, fill=laneA, inner sep=1.5pt] at (6.90,0.38)
  {this cycle's outcomes};
\end{tikzpicture}

\caption{
Observed disagreement in \textsc{Discord} motivates predictive allocation in \methodname{}.
\emph{(a)} Paired disagreements determine retention evidence. Even two additional
candidate wins cannot pass the toy example's $\alpha=0.05$ test, so evaluation stops.
\emph{(b)} \methodname{} uses \trCap{1} historical outcomes to
\trCap{2} forecast reward unanimity and \trCap{3} weight eligible tasks before
\trCap{4} generating one fixed-size group per selected task. Chips show relative
sampling weights. The panels depict observed reward contrast for harness
retention and forecast reward contrast for task allocation; all-equal rollout
groups supply no group-relative reward signal.
}
\label{fig:overview}
\end{figure}
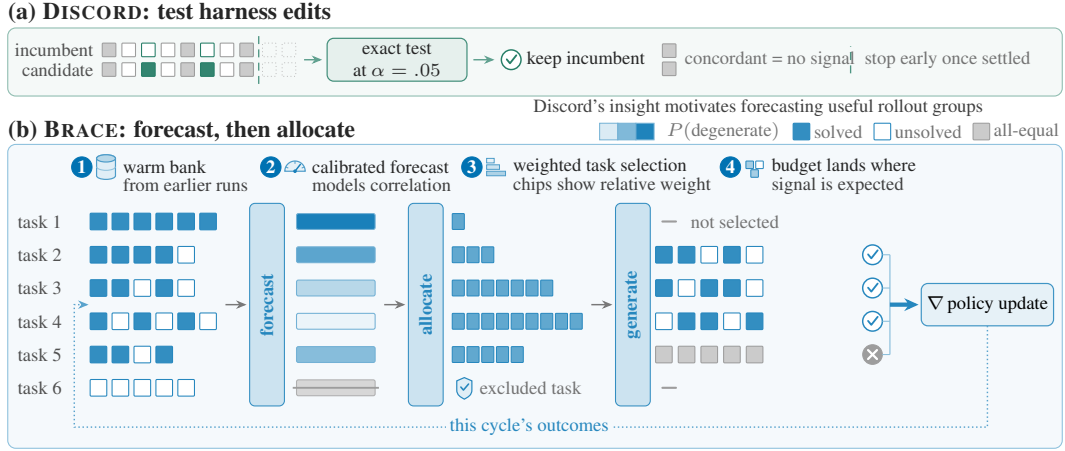

\section{Introduction}
Language agents use tools over multiple turns before a task's outcome can be
verified \citep{zhou2024archer,wang2026agent}. The Model Context Protocol (MCP)
standardizes how agent applications discover and invoke external tools and
access resources \citep{hou2026mcpsurvey,anjum2026promcp}. An \emph{agent harness} is the
model-external software that configures prompts and tools and mediates the
action--observation loop \citep{lee2026metaharness,yang2024sweagent}. Together with the model policy,
it forms the agent: changing prompts or tool access changes the harness,
whereas training model weights changes the policy.
Automated search can improve the harness
\citep{hu2025adas,zhang2025aflow,lee2026metaharness}, but deciding whether to
retain an edit requires evaluating its effect across costly agent trajectories.
An apparent gain on a finite task set may reflect noise. When the candidate and
incumbent run on the same tasks, agreements cannot favor either one; only
disagreements reveal the direction of the change. This motivates
\textsc{Discord}, which tests paired outcomes and stops evaluation when the
remaining tasks cannot change the retention decision.

A similar allocation problem arises when training that policy. Reinforcement
learning with verifiable rewards (RLVR) uses task
outcomes from complete trajectories \citep{guo2025deepseekr1,wang2026agent}.
Group Relative Policy Optimization (GRPO) \citep{shao2024deepseekmath} compares
several rollouts of the same task, but a group in which all rewards agree has
zero centered reward advantage. Generating such groups is especially costly
when each rollout may involve repeated model inference and MCP tool calls.
DAPO discards them after generation \citep{yu2025dapo}; PLR prioritizes tasks
by estimated learning potential \citep{jiang2021plr}, and VIP, TRACE, and
DUET allocate generation prospectively
\citep{nguyen2026vip,zou2026trace,hu2026duet}. For multi-turn tool-use tasks,
we ask whether outcomes from earlier runs can predict which fixed-size groups
will contain mixed rewards, despite uncertain task success rates and correlated
rollouts. \methodname{} uses that forecast to allocate generation before the
next group is sampled, forming an adaptive curriculum based on the expected
availability of a group-relative learning signal.

We study two complementary routes to improving tool-using agents.
\textsc{Discord} uses evaluation compute to decide whether to retain a
harness edit without changing model weights. Its focus on informative disagreement motivates \methodname{}
to forecast which training tasks will produce mixed-reward groups and allocate
rollouts accordingly. This directs training compute toward the policy operating
within the harness (\hyperref[fig:overview]{Figure~\ref*{fig:overview}}).
Our contributions are:
\begin{itemize}
\item \textbf{Reliable harness retention.} \textsc{Discord} uses exact paired
inference and decision-preserving stopping; controlled and live experiments
measure its recovery and evaluation cost.
\item \textbf{Prospective task selection.} \methodname{} combines bounded
historical evidence, uncertainty over task success, and within-group dependence
to forecast informative groups. We characterize the resulting objective and
test its calibration and components.
\item \textbf{Training efficiency and transfer.} Across three model scales,
\methodname{} improves held-out tool-use performance and reaches the strongest
plotted baseline peaks with fewer new rollouts. It also transfers to BFCL
\citep{patil2025bfcl} and NESTFUL \citep{basu2024nestful} without further training.
\end{itemize}

\section{Problem Formulation}
\label{sec:formulation}
\subsection{Group degeneracy in RLVR}
\label{sec:formulation-degen}
For task $i$, let $r_{i1},\ldots,r_{ik}\in\{0,1\}$ be verifier rewards for $k$
rollouts. The centered reward component of the GRPO \citep{shao2024deepseekmath} advantage is
\begin{equation}
A_{ij}=r_{ij}-\bar r_i,\qquad
\bar r_i=\frac{1}{k}\sum_{j'=1}^{k}r_{ij'}.
\label{eq:group-advantage}
\end{equation}
Any within-group normalization leaves the following degeneracy condition unchanged:
\begin{equation}
r_{i1}=\cdots=r_{ik}\quad\Longrightarrow\quad A_{ij}=0\quad\forall j.
\label{eq:degen}
\end{equation}
This condition concerns the reward-based policy-gradient term; other regularizers
may still contribute to an update. Non-degeneracy is necessary for that term to be
informative, but does not guarantee a beneficial update.
We quantify degeneracy in a bank snapshot of $482{,}510$ rollouts on $1{,}493$ verifier-scored tasks from Agent
World Model (AWM) \citep{wang2026agent}, a suite of SQLite-backed MCP environments
scored against the final database state. It includes the $1{,}127$-task training pool
and additional tasks from the same substrate. Reconstructing original optimizer
steps and retaining groups with at least two rollouts yields $92{,}706$ groups,
of which $85.2\%$ are degenerate.

\begin{algorithm}[tb]
\caption{\methodname{} allocation cycle: forecast and select tasks before generation.}
\label{alg:brace}

{\footnotesize
\noindent
\textbf{Inputs:}
task pool $\mathcal{T}$, excluded set $\mathcal{D}$, counts $(s_i,f_i)$,
bank $\mathcal{B}$, group size $k$, task budget $N\le|\mathcal{T}\setminus\mathcal{D}|$, discount $\gamma$,
correlation $\hat{\rho}$, temperature $\tau$, exploration rate $\epsilon$.

\begin{enumerate}\itemsep1pt\parskip0pt\topsep2pt
\item
\textbf{Update evidence:}
on the first cycle, set $(s_i,f_i)$ to capped counts from $\mathcal{B}$.
On later cycles, apply
$(s_i,f_i)\leftarrow\gamma(s_i,f_i)+(\Delta s_i,\Delta f_i)$
using the preceding cycle's outcomes exactly once.

\item
\textbf{Forecast informative groups:}
for each $i\in\mathcal{T}\setminus\mathcal{D}$, use
$p_i\sim\mathrm{Beta}(1+s_i,1+f_i)$ and
$\nu=(1-\hat{\rho})/\hat{\rho}$ to compute
$w_i=1-\mathbb{E}[q(p_i)]-\mathbb{E}[q(1-p_i)]$
via Equations~\ref{eq:bb} and~\ref{eq:corr-w}.

\item
\textbf{Allocate budget:}
sample $N-n_\epsilon$ tasks without replacement with probability
$\propto w_i^{1/\tau}$, where
$n_\epsilon=\mathrm{round}(\epsilon N)$, and sample the remaining
$n_\epsilon$ uniformly from unselected eligible tasks.
Tasks in $\mathcal{D}$ receive no budget.

\item
\textbf{Generate and update:}
generate one group of $k$ rollouts for each selected task, perform the scheduled policy updates,
and record $(\Delta s_i,\Delta f_i)$ for the next cycle's evidence update.
\end{enumerate}
}
\end{algorithm}

\begin{figure}[tb]
\centering
\includegraphics[width=\textwidth]{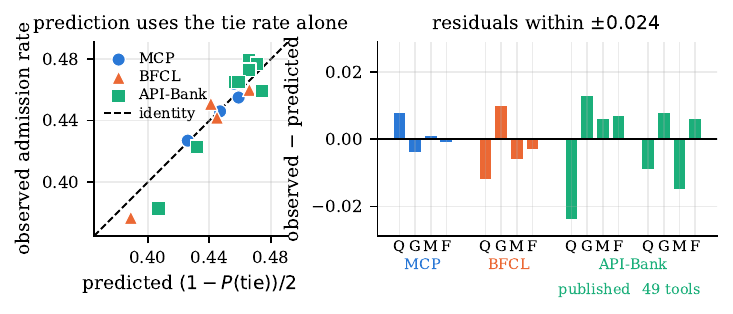}
\caption{Inert-edit admission across AWM \citep{wang2026agent}, BFCL \citep{patil2025bfcl}, and two API-Bank \citep{li2023apibank} tool surfaces.
Left: predictions from independently measured tie rates versus observed admission.
Right: residuals, all within $\pm0.024$ (Appendix~\ref{app:identity}).}
\label{fig:identity}
\end{figure}

\subsection{Paired outcomes in harness evolution}
\label{sec:formulation-paired}

Harness evolution compares an incumbent configuration $A$ with a candidate $B$ by
evaluating both on the same tasks. Pairing controls for task difficulty and allows the
decision to depend directly on how the two configurations differ on identical inputs.
For binary outcomes $(A_i,B_i)\in\{0,1\}^2$, define the two discordant counts
\begin{equation}
b=\#\{i:A_i=0,B_i=1\},\qquad
c=\#\{i:A_i=1,B_i=0\}.
\label{eq:bc}
\end{equation}
The paired difference in mean reward can then be written as
\begin{equation}
\bar B-\bar A=\frac{b-c}{n},
\end{equation}
because pairs with $(A_i,B_i)=(0,0)$ or $(1,1)$ cancel exactly. Thus, only
discordant pairs favor one harness over the other; $n_d=b+c$ counts informative
pairs.

These identities distinguish two statistical objects. \textsc{Discord} uses
paired disagreements to support a harness-retention decision; \methodname{}
forecasts whether a same-task group will mix successes and failures, enabling a
group-relative reward signal. The shared principle is to allocate interactions
toward outcomes that can affect the next decision or update.

\section{\textsc{Discord}: Exact Harness Retention}
\label{sec:retention}
\label{sec:seq}
\textsc{Discord} evaluates proposed harness edits, holding model weights fixed.
It retains an edit only when paired outcomes provide sufficient evidence of
improvement; it does not generate proposals.

\begin{figure}[tb]
\centering
\begin{minipage}[t]{0.51\textwidth}
\vspace{0pt}
\centering
\includegraphics[width=\linewidth]{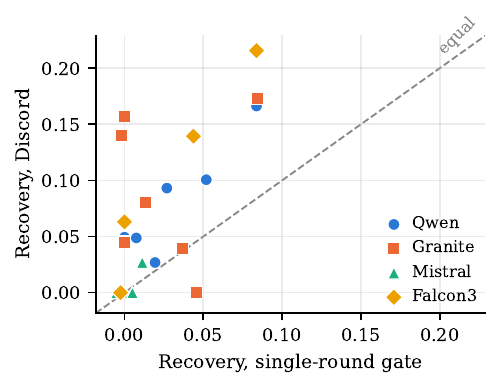}
\caption{Harness recovery across four callers: $18/21$ paired loops favor
\textsc{Discord}. The dashed line marks equal recovery.}
\label{fig:div}
\end{minipage}\hfill
\begin{minipage}[t]{0.47\textwidth}
\vspace{0pt}
\captionof{table}{Live AWM \citep{wang2026agent} retention and cost. Accepted edits are classified as
good/bad by measured effect; objective and episode counts report realized outcomes.}
\label{tab:discord-live}
\centering\footnotesize
\setlength{\tabcolsep}{2pt}
\begin{tabular*}{\linewidth}{@{\extracolsep{\fill}}llrrr@{}}
\toprule
Regime & Rule & \shortstack{Accepted\\(good/bad)} & \shortstack{Final\\objective} & Episodes \\
\midrule
Shallow & Deployed & $3$ ($2/1$) & $0.2656$ & $1{,}760$ \\
 & \textsc{Discord} & $0$ ($0/0$) & $0.2594$ & $1{,}906$ \\
Deep & Deployed & $3$ ($3/0$) & $0.2812$ & $1{,}920$ \\
 & \textsc{Discord} & $2$ ($2/0$) & $0.2703$ & $4{,}008$ \\
\bottomrule
\end{tabular*}

\par\smallskip
\raggedright
Shallow and deep regimes withhold $10\%$ and $40\%$ of tool names.
The full-surface objective is $0.2813$.
\end{minipage}
\end{figure}

\subsection{Why a positive measured gain is insufficient}
\label{sec:chance}
An \emph{inert edit} leaves true performance unchanged. If incumbent and candidate
scores are exchangeable, apparent improvements and regressions are equally likely:
\begin{equation}
P(\mathrm{retain})=P(\bar B>\bar A)
=\tfrac12\bigl(1-P(\mathrm{tie})\bigr).
\label{eq:chance}
\end{equation}
When ties are rare, strict-improvement retention accepts nearly half of inert
edits. Appendix~\ref{app:identity} gives the derivation and measurement protocol.

\subsection{Testing whether the candidate wins reliably}
\label{sec:discord-test}
Under the paired null, a disagreement is equally likely to favor the candidate
or the incumbent. Assuming independent evaluation pairs and conditioning on the
observed number of disagreements $n_d=b+c$, the number of candidate wins follows
McNemar's exact conditional distribution \citep{mcnemar1947}:
\begin{equation}
b\mid n_d\sim\mathrm{Binomial}(n_d,1/2).
\label{eq:discord-null}
\end{equation}
\textsc{Discord} retains the candidate when the upper-tail probability is at
most a prespecified $\alpha$. For example, $b=12$ and $c=4$ on $100$ paired
tasks give an $8$ pp improvement and an exact one-sided probability of $0.0384$,
passing at $\alpha=0.05$. With no disagreements, retention is rejected.
Operationally, each paired task either increments $b$, increments $c$, or ties
and changes neither count. Conditioning on $n_d$ removes the unknown tie rate:
the test asks whether the directions of observed disagreements favor the
candidate, independently of how often the harnesses agree.
Error control applies to this prespecified comparison, not an unrestricted
adaptive search (Appendix~\ref{app:gatefull}).

\label{sec:discord-budget}
\noindent\textbf{Stopping without changing the decision.}
The maximum number of paired tasks is fixed in advance. After each block,
\emph{exact curtailment} \citep{alling1963curtailed} checks whether all possible
remaining pairs (candidate wins, incumbent wins, or ties) lead to the same
terminal decision. If so, evaluation stops; otherwise it continues.

\begin{samepage}
\noindent For example, with $b=2$, $c=0$ and two pairs remaining, the most
favorable completion for the candidate is
$b=4$, $c=0$. Its one-sided probability is $1/16=0.0625$, so it still fails an
$\alpha=0.05$ test; the procedure can immediately keep the incumbent. This is
different from repeatedly stopping when a provisional $p$-value happens to cross
the threshold: exact curtailment always returns the decision that full evaluation
would have returned. Ambiguous comparisons can still require the entire budget
(Appendices~\ref{app:gatefull} and~\ref{app:seq}).
\end{samepage}

\noindent\textbf{From observed to predicted disagreement.}
\textsc{Discord} identifies when further paired outcomes cannot change a
harness decision. This suggests a training question: which tasks are likely to
yield mixed rewards before their rollouts are generated? \methodname{} forecasts
that signal from historical task outcomes and allocates the next budget
accordingly. We evaluate its forecasts and policy gains separately.

\begin{figure}[tb]
\centering
\includegraphics[width=\textwidth]{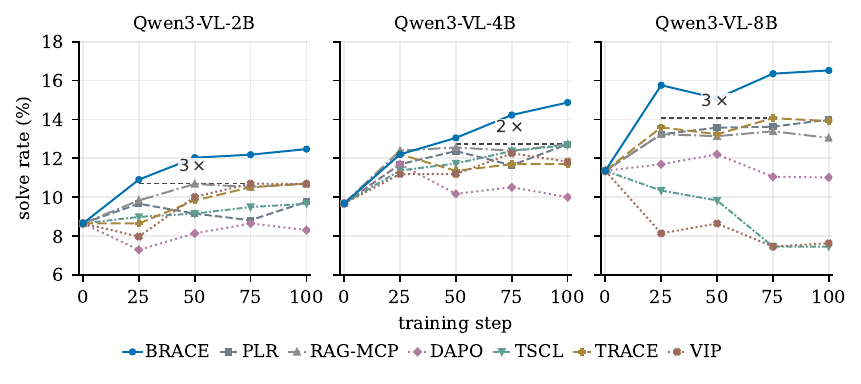}
\caption{Observed-checkpoint comparison of new rollouts ($160$ attempts/step).
For each scale, the horizontal dashed segment spans the earliest peak checkpoint
of the strongest plotted baseline and BRACE's first measured checkpoint that
meets or exceeds that solve rate. Labels give the baseline/BRACE rollout ratio
at those checkpoints; bank collection is excluded.
\hyperref[tab:main]{Table~\ref*{tab:main}} reports final-checkpoint results.}
\label{fig:stepcurve}
\end{figure}

\section{\methodname{}: Correlation-Aware Rollout Allocation}
\label{sec:allocation}
\label{sec:brace-alg}

\methodname{} selects tasks before generating their next rollout groups
(\hyperref[alg:brace]{Algorithm~\ref*{alg:brace}}). It favors tasks predicted to produce both successes
and failures while retaining uniform exploration. Main experiments select $N=256$
distinct tasks per allocation cycle, then generate one $k=5$ group per task over
eight optimizer steps of $32$ tasks each. Allocation changes task selection, not
group size. Thus, the score targets the chance of receiving a group-relative
reward signal, rather than the highest expected solve rate; the exploration
share continues to sample eligible tasks regardless of their scores.
\hyperref[alg:brace]{Algorithm~\ref*{alg:brace}} records each selected batch and its predicted
degeneracy before generation, enabling prospective calibration checks.

\subsection{Allocation Objective and Banked Evidence}
\label{sec:brace-obj}
\label{sec:brace-bank}

For task solve probability $p_i$, an independent group is degenerate if all $k$
rollouts succeed or all fail. The complementary probability of mixed rewards is
\begin{equation}
g(p_i,k)=1-p_i^k-(1-p_i)^k.
\label{eq:availability}
\end{equation}
Under independence, we account for uncertainty in $p_i$ by integrating over its posterior:
\begin{equation}
w_i^{\mathrm{ind}}=\mathbb E[g(p_i,k)],\qquad p_i\sim\mathrm{Beta}(a_i,b_i),
\quad \mathbb E[p_i^k]=\prod_{j=0}^{k-1}\frac{a_i+j}{a_i+b_i+j}.
\label{eq:w}
\end{equation}
Tasks with the same posterior mean can receive different scores when their
evidence counts differ.

\textbf{Banked evidence.} We approximate uncertainty in the task mean by
$p_i\sim\mathrm{Beta}(1+s_i,1+f_i)$ using success and failure counts $(s_i,f_i)$;
this is not the exact posterior under the correlated group likelihood.
A \emph{warm bank} provides success and failure counts from previous runs across
model scales, capped so a large historical log cannot dominate later evidence.
It is read once to initialize task-level evidence for allocation. Subsequently,
the counts are discounted and incremented with current-run outcomes; the bank
is not reread. Every rollout used for a policy update is generated by the
current policy. Without bank observations,
unseen tasks all start at $\mathrm{Beta}(1,1)$, so
allocation cannot initially distinguish tasks that rarely solve from those that
produce mixed outcomes. Bank counts make that distinction available before the
first current-run rollout; capping limits their influence as new evidence arrives.
Appendix~\ref{app:bank-provenance} gives the exact initialization and cost accounting.

\subsection{Modeling Correlated Rollout Groups}
\label{sec:brace-rho}

Rollout rewards within the same task group need not be independent: shared task
difficulty and sampling conditions can make the outcomes move together. We model
this dependence with an exchangeable Beta--Binomial distribution
with mean $p_i$ and concentration $\nu=(1-\rho)/\rho$, where $\rho$ is the
within-group correlation. Equivalently, for each group we draw a latent success
propensity $\theta_i\sim\mathrm{Beta}(\nu p_i,\nu(1-p_i))$, then draw the $k$
rewards conditionally independently given $\theta_i$. This group-level draw is
distinct from the posterior uncertainty over the task mean $p_i$. Under this model,
\begin{equation}
q(p_i)=P(\text{all }k\text{ succeed}\mid p_i)
=\prod_{j=0}^{k-1}\frac{\nu p_i+j}{\nu+j},
\label{eq:bb}
\end{equation}
and $q(1-p_i)$ is the probability of unanimous failure. The allocation score is
\begin{equation}
w_i=1-\mathbb E[q(p_i)]-\mathbb E[q(1-p_i)].
\label{eq:corr-w}
\end{equation}
Under the Beta evidence approximation and exchangeable group model, $w_i$ is
the posterior predictive probability that task $i$'s next group contains both
successes and failures.
For one group per selected task, $\sum_{i\in S}w_i$ is the model's expected
informative-group count (Appendix~\ref{app:yield-identity}). This is a
signal-availability objective, not a policy-improvement guarantee.
The independent case is recovered as $\nu\to\infty$. Positive correlation raises
the probability of unanimous groups; for example, at $p_i=0.5$, $k=5$, and
$\nu=0.285$, the mixed-reward probability is $0.23$, compared with $0.94$ under
independence. The expectations in Equation~\ref{eq:corr-w} are evaluated exactly
using Beta moments, without sampling (Appendix~\ref{app:rhofull}).

\subsection{Rank Preservation under Correlation Correction}
\label{sec:brace-rank}

\begin{proposition}
\label{prop:rankcal}
For $k\ge2$ and $\nu>0$, let
$f(p)=\prod_{j=0}^{k-1}(\nu p+j)/(\nu+j)$.
Then $g_c(p)=1-f(p)-f(1-p)$ is strictly increasing in $p(1-p)$.
Consequently, at fixed $k$ and shared $\nu$, known-probability scores $g_c(p)$,
$1-p^k-(1-p)^k$, and the Bernoulli variance $p(1-p)$ induce the same task ordering.
\end{proposition}
For the weighted sampling step, the relative weight of eligible tasks $i$ and $j$
is $(w_i/w_j)^{1/\tau}$. A common rescaling of all scores cancels, while a nonlinear
transformation can change these ratios even when ranks agree. Posterior integration
can also change rankings because evidence counts differ across tasks. The mechanism
therefore depends on the estimated evidence and relative weights, alongside forecast
calibration. Appendix~\ref{app:proofs} proves the proposition;
Section~\ref{sec:components} connects it to sampling controls and training ablations.
The theorem concerns fixed group sizes.

\begin{table}[tb]
\centering\footnotesize
\renewcommand{\arraystretch}{0.95}
\setlength{\tabcolsep}{2.5pt}
\caption{Held-out AWM \citep{wang2026agent} results after $100$ training steps. Higher solve rate,
lower invalid-call rate, and fewer turns are better. Bold marks the best trained
method per metric and scale; evaluation uses four seeds.}
\vspace{4pt}

\begin{tabular*}{\textwidth}{@{\extracolsep{\fill}}l rrr rrr rrr@{}}
\toprule
 & \multicolumn{3}{c}{Qwen3-VL-$2$B} & \multicolumn{3}{c}{Qwen3-VL-$4$B} & \multicolumn{3}{c}{Qwen3-VL-$8$B} \\
\cmidrule(lr){2-4}\cmidrule(lr){5-7}\cmidrule(lr){8-10}
Method & solve & invalid & turns & solve & invalid & turns & solve & invalid & turns \\
 & (\%)\,$\uparrow$ & (\%)\,$\downarrow$ & $\downarrow$ & (\%)\,$\uparrow$ & (\%)\,$\downarrow$ & $\downarrow$ & (\%)\,$\uparrow$ & (\%)\,$\downarrow$ & $\downarrow$ \\
\midrule
\hspace{4pt} base model & $8.6$ & $16.6$ & $14.7$ & $9.7$ & $6.5$ & $8.4$ & $11.4$ & $6.8$ & $9.6$ \\
\hspace{4pt}\methodname{} & $\mathbf{12.5}$ & $19.3$ & $12.1$ & $\mathbf{14.9}$ & $4.6$ & $10.0$ & $\mathbf{16.5}$ & $9.0$ & $9.0$ \\
\hspace{4pt}\textsc{Plr} \citep{jiang2021plr} & $9.8$ & $20.0$ & $\mathbf{7.6}$ & $12.7$ & $6.9$ & $11.3$ & $14.0$ & $9.1$ & $10.2$ \\
\hspace{4pt}\textsc{Rag}-\textsc{Mcp} \citep{gan2025ragmcp} & $10.7$ & $29.4$ & $11.6$ & $12.7$ & $\mathbf{3.0}$ & $10.9$ & $13.1$ & $\mathbf{3.6}$ & $11.6$ \\
\hspace{4pt}\textsc{Dapo} \citep{yu2025dapo} & $8.3$ & $18.9$ & $14.2$ & $10.0$ & $6.3$ & $11.6$ & $11.0$ & $8.0$ & $12.3$ \\
\hspace{4pt}\textsc{Tscl} \citep{matiisen2020tscl} & $9.7$ & $\mathbf{10.4}$ & $10.3$ & $12.7$ & $5.5$ & $8.4$ & $7.5$ & $9.3$ & $\mathbf{3.7}$ \\
\hspace{4pt}\textsc{Trace} \citep{zou2026trace} & $10.7$ & $30.2$ & $9.1$ & $11.7$ & $6.9$ & $13.0$ & $13.9$ & $6.6$ & $10.3$ \\
\hspace{4pt}\textsc{Vip} \citep{nguyen2026vip} & $10.7$ & $14.3$ & $13.6$ & $11.8$ & $5.0$ & $\mathbf{6.7}$ & $7.6$ & $6.4$ & $11.0$ \\
\hspace{4pt}uniform \textsc{Grpo} \citep{shao2024deepseekmath} & $8.9$ & $29.9$ & $11.2$ & $11.4$ & $9.8$ & $10.1$ & $13.2$ & $13.0$ & $8.7$ \\
\bottomrule
\end{tabular*}
\par\addvspace{3pt}

\label{tab:main}
\end{table}

\section{Experiments}
\label{sec:triage-exp}

\subsection{Experimental setup}
\label{sec:triage-protocol}

\textbf{Models and baselines.}
We train Qwen3-VL-Instruct \citep{bai2025qwen3vl} at $2$B, $4$B, and $8$B using
GRPO \citep{shao2024deepseekmath} in \texttt{verl} \citep{sheng2025hybridflow}, with rank-$32$ LoRA.
Every method generates $160$ attempts per training step
($32$ prompts, $k=5$), counting discarded and incomplete attempts.
The AWM \citep{wang2026agent} training pool contains $1{,}127$ verifier-scored tasks
over $305$ live MCP servers under a fixed tool surface. Baselines include uniform
GRPO \citep{shao2024deepseekmath}, PLR \citep{jiang2021plr}, RAG-MCP \citep{gan2025ragmcp},
DAPO \citep{yu2025dapo}, TSCL \citep{matiisen2020tscl}, TRACE \citep{zou2026trace},
and VIP \citep{nguyen2026vip} (\hyperref[tab:main]{Table~\ref*{tab:main}}). Methods
share the model, optimizer, current-run rollout budget, and evaluation protocol;
\methodname{} additionally uses historical bank evidence.

\noindent\textbf{Eligibility checks.} We exclude $104$ of $1{,}127$ training tasks: $66$ whose verifier succeeds on the
unchanged initial state, $33$ proved unreachable by an initial-state guard, and
$5$ flagged by a string-matching heuristic, reported separately from certificates.
The guard never fires on the bank's $826$ tasks solved at least once; excluded
tasks produce no usable learning groups in historical data. These checks determine
eligibility before posterior weighting: failures alone do not prove impossibility.

\noindent\textbf{BRACE configuration.}
Historical counts are capped at $20$ successes and $40$ failures per task, with
$\gamma=1$. We use $\nu=0.285$ ($\hat\rho\approx0.78$) and $\tau=0.5$ at all
scales; uniform exploration is $\epsilon=0.05$ at $2$B/$4$B and $0.10$ at $8$B.
The correlation parameter is fitted once on the first two cycles of the $8$B
calibration run and then held fixed. Those cycles are excluded from forecast
evaluation. Appendix~\ref{app:bank-provenance} documents bank provenance and costs.

\noindent\textbf{Evaluation.} We evaluate on $295$
held-out AWM \citep{wang2026agent} tasks, excluded from training and bank initialization, using four
evaluation seeds. BFCL v4 \citep{patil2025bfcl} and NESTFUL
\citep{basu2024nestful} supply no AWM \citep{wang2026agent} training data. Main AWM \citep{wang2026agent} results use step
$100$; paired contrasts use matched task indices. 

\subsection{\textsc{Discord}: Harness Retention}
\label{sec:evo-main}
\label{sec:mcp-live}

\noindent\textbf{Inert-edit admission.}
Independently measured tie rates predict inert-edit admission within $\pm0.024$
across AWM \citep{wang2026agent}, BFCL \citep{patil2025bfcl}, and two API-Bank
\citep{li2023apibank} tool surfaces (\hyperref[fig:identity]{Figure~\ref*{fig:identity}}). On AWM
\citep{wang2026agent}, strict-improvement retention accepts $43.3\%$ of
$1{,}440$ inert edits, and $43.6\%$ of the accepted edits fail to replicate
(Appendices~\ref{app:gatefull} and~\ref{app:identity}).

\begin{figure}[!tb]
\centering
\includegraphics[width=\textwidth]{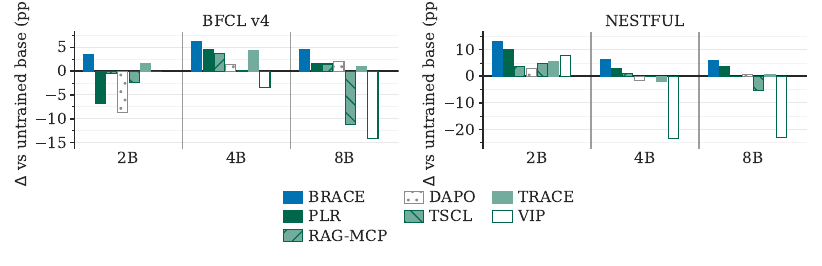}
\caption{Transfer to BFCL \citep{patil2025bfcl} and NESTFUL
\citep{basu2024nestful} without further training. Bars show percentage-point
changes from the corresponding untrained base model, not absolute rates.
\methodname{} has the largest gain among the plotted methods at each scale.}
\label{fig:generalization}
\end{figure}

\noindent\textbf{Controlled recovery and harmful proposals.}
We compare retention rules on $200$ BFCL v3 \citep{patil2025bfcl} single-turn tasks,
using four callers and eight rounds of three edits to
advertised function names. Withheld names create a recoverable handicap; a full
surface makes proposals predominantly harmful. Shared proposals and matched episode
budgets isolate retention from edit generation (Appendix~\ref{app:loopfull}).
The single-round gate uses a label-permutation test on disjoint task halves;
\textsc{Discord} uses discordant outcomes from the same tasks.
\textsc{Discord} recovers more of the handicap than the single-round gate in
$18$ of $21$ paired loops (\hyperref[fig:div]{Figure~\ref*{fig:div}}; one-sided paired
$t$-test, $p=0.00015$). Under harmful proposals, degradation occurs in $0/17$
seeds for \textsc{Discord} and $9/17$ for strict-improvement retention. The seeds
are matched; a two-sided exact McNemar test gives $p=0.0039$.

\noindent\textbf{Live MCP trade-off.}
On a shared proposal stream across $320$ tasks and $158$ AWM \citep{wang2026agent} servers,
shallow and deep handicaps withhold $10\%$ and $40\%$ of
tool names. \textsc{Discord} accepts no harmful edits in either regime, but uses
more evaluation episodes and reaches a lower final objective than the deployed
rule (\hyperref[tab:discord-live]{Table~\ref*{tab:discord-live}}). The deployed rule overstates its final
objective by $32\%$ and $47\%$, respectively. This trade-off favors
\textsc{Discord} when avoiding harmful edits warrants more evaluation; conservative
retention can miss useful edits under limited evidence.

\subsection{\texorpdfstring{Held-Out MCP Performance}{Held-Out MCP Performance}}
\label{sec:triage-results}

\methodname{} leads all baselines at $2$B, $4$B, and $8$B
(\hyperref[tab:main]{Table~\ref*{tab:main}}), solving $12.5\%$, $14.9\%$, and $16.5\%$ of tasks.
Uniform GRPO \citep{shao2024deepseekmath} reaches $8.9\%$, $11.4\%$, and $13.2\%$;
\methodname{} gains $3.6$, $3.5$, and $3.3$ percentage points (pp).
The strongest baselines reach $10.7\%$, $12.7\%$, and $14.0\%$, giving margins
of $1.8$, $2.2$, and $2.5$ pp. Section~\ref{sec:components} separates bank and
allocation effects.
At $2$B, $4$B, and $8$B, BRACE exceeds the strongest plotted baseline peaks
using $4{,}000$, $8{,}000$, and $4{,}000$ new rollouts, versus $12{,}000$,
$16{,}000$, and $12{,}000$ at the baselines' earliest peak checkpoints.
These observed comparisons require no interpolation.
\hyperref[fig:stepcurve]{Figure~\ref*{fig:stepcurve}} plots solve rate against training
steps; dashed segments span the measured checkpoints used for the rollout ratios. RAG-MCP \citep{gan2025ragmcp} has the lowest invalid-call rates at $4$B and $8$B;
PLR \citep{jiang2021plr}, VIP \citep{nguyen2026vip}, and TSCL
\citep{matiisen2020tscl} use the fewest turns at $2$B, $4$B, and $8$B.
\hyperref[tab:main]{Table~\ref*{tab:main}} thus shows better task completion without
consistently shorter trajectories for \methodname{}.

\subsection{Generalization}
\label{sec:generalization}

Transfer tests multi-turn function calling on BFCL \citep{patil2025bfcl} and
nested function composition on NESTFUL \citep{basu2024nestful}, without further
adaptation. \hyperref[fig:generalization]{Figure~\ref*{fig:generalization}} shows
percentage-point gains over untrained bases; \methodname{} leads at all three scales.
NESTFUL \citep{basu2024nestful} win rate evaluates the executed answer,
not exact call-sequence matching. Our modified prompt and parsing conventions
support within-study comparisons, not direct comparison with published NESTFUL \citep{basu2024nestful}
scores.

\begin{table}[tb]

\centering\footnotesize
\renewcommand{\arraystretch}{0.95}
\setlength{\tabcolsep}{2.5pt}
\caption{Component comparisons and a uniform-GRPO \citep{shao2024deepseekmath} reference on held-out AWM \citep{wang2026agent} tasks. ``$\Delta$ vs full'' and
``effect vs base'' are percentage-point differences computed from the displayed
solve rates. The full method has the
highest solve rate at all three model scales.}
\vspace{4pt}
\begin{tabular*}{\textwidth}{@{\extracolsep{\fill}}>{\raggedright\arraybackslash}p{0.38\textwidth} rrr rrr rrr@{}}
\toprule
 & \multicolumn{3}{c}{Qwen3-VL-$2$B} & \multicolumn{3}{c}{Qwen3-VL-$4$B} & \multicolumn{3}{c}{Qwen3-VL-$8$B} \\
\cmidrule(lr){2-4}\cmidrule(lr){5-7}\cmidrule(lr){8-10}
ablation & solve & $\Delta$ vs & effect & solve & $\Delta$ vs & effect & solve & $\Delta$ vs & effect \\
 & (\%) & full & vs base & (\%) & full & vs base & (\%) & full & vs base \\
\midrule
\methodname{} (full) & $\mathbf{12.5}$ & -- & $+3.9$ & $\mathbf{14.9}$ & -- & $+5.2$ & $\mathbf{16.5}$ & -- & $+5.1$ \\
\quad w/o warm bank & $9.2$ & $-3.3$ & $+0.6$ & $12.8$ & $-2.1$ & $+3.1$ & $12.2$ & $-4.3$ & $+0.8$ \\
\quad w/o correlation correction & $10.0$ & $-2.5$ & $+1.4$ & $10.3$ & $-4.6$ & $+0.6$ & $14.4$ & $-2.1$ & $+3.0$ \\
\quad estimator: posterior $\to$ plug-in & $10.7$ & $-1.8$ & $+2.1$ & $13.1$ & $-1.8$ & $+3.4$ & $15.2$ & $-1.3$ & $+3.8$ \\
\quad correlation $\rho$: one $\to$ per-scale & $11.3$ & $-1.2$ & $+2.7$ & $13.9$ & $-1.0$ & $+4.2$ & $15.6$ & $-0.9$ & $+4.2$ \\
\quad score $\to$ \textsc{Vip}-style variance & $10.6$ & $-1.9$ & $+2.0$ & $13.2$ & $-1.7$ & $+3.5$ & $14.1$ & $-2.4$ & $+2.7$ \\
\quad reference: uniform \textsc{GRPO} & $8.9$ & $-3.6$ & $+0.3$ & $11.4$ & $-3.5$ & $+1.7$ & $13.2$ & $-3.3$ & $+1.8$ \\
\bottomrule
\end{tabular*}
\par\addvspace{3pt}

\label{tab:configs}
\end{table}

\subsection{Ablation Study}
\label{sec:components}

\hyperref[tab:configs]{Table~\ref*{tab:configs}} shows lower solve rates in every component ablation
at all three scales. Removing the bank costs up to $4.3$ pp; removing correlation
correction costs up to $4.6$ pp. Plug-in estimation, per-scale correlation fitting,
and the VIP-style \citep{nguyen2026vip} variance score also underperform. The score comparison
holds bank counts, estimator, temperature, exploration, and eligibility fixed.
\textbf{Mechanism and cost controls.} The bank reuses episodes from earlier $2$B,
$4$B, and $8$B runs. It is read once
at initialization; no dedicated collection run was launched, although the
historical generation cost remains separate from current-run rollouts. The
VIP-score \citep{nguyen2026vip} and correlation-removal ablations start from the same capped
per-task bank counts as full \methodname{} (Appendix~\ref{app:bank-provenance}).
On $83$ prospective arm-cycles that exclude the cycles used to fit $\rho$,
correlation-aware forecasts reduce mean absolute degeneracy-prediction error
from $0.591$ to $0.056$ (Appendix~\ref{app:calib}). This tests forecast
calibration, while held-out solve rates assess learning; Proposition~\ref{prop:rankcal}
also limits what calibration alone can imply about task ranking.

\section{Related Work}
\label{sec:triage-related}

\noindent\textbf{Task selection and rollout allocation.}
Adaptive curricula change which tasks an agent encounters. PLR prioritizes
levels using TD error \citep{jiang2021plr}, while problem-level RL replay uses
empirical intermediate success rates \citep{fatemi2026prioritized}. Other
selectors use reward dynamics \citep{zheng2025greso}, task difficulty
\citep{qu2025mopps,qu2026gps,gao2025pcl}, estimated learning impact
\citep{li2025limr}, or replay priorities \citep{sun2025dots}. For group-relative
training, DAPO filters all-equal groups after generation \citep{yu2025dapo}.
VIP allocates rollout counts by predicted gradient variance
\citep{nguyen2026vip}; DUET also allocates token budgets \citep{hu2026duet},
and TRACE predicts mixed outcomes at prompts and interaction prefixes using
an independent rollout model \citep{zou2026trace}. \methodname{} instead selects
tasks for fixed-size groups using capped historical evidence, posterior
uncertainty, and a correlated reward model. The correlation correction improves
forecast calibration while preserving known-probability task ranks at fixed
group size (Proposition~\ref{prop:rankcal}); historical evidence and relative
weights therefore matter to the resulting allocation. Policy-gradient
estimators \citep{liu2025drgrpo,ahmadian2024rloo,zheng2025gspo} address a
separate choice.

\noindent\textbf{Tool-using agents and harness search.}
Multi-turn training studies delayed rewards and tool interaction
\citep{zhou2024archer,wang2025ragen,jin2025searchr1,qian2025toolrl}. MCP work
studies protocol architecture and overhead, tool discovery, and agent evaluation
\citep{hou2026mcpsurvey,anjum2026promcp,fei2025mcpzero,guo2026mcpagentbench,wang2025mcpbench}, while
benchmarks measure tool execution and function composition
\citep{patil2025bfcl,basu2024nestful,qin2023toolllm,yao2024taubench}.
Harness-search systems propose edits to prompts, workflows, or agent code
\citep{hu2025adas,zhang2025aflow,agrawal2026gepa,lee2026metaharness,zelikman2024stop,zhang2026dgm}.
\textsc{Discord} addresses the subsequent retention decision with exact
conditional paired inference \citep{mcnemar1947} and decision-preserving
curtailment \citep{alling1963curtailed}; alternative matched-pair tests include
exact unconditional procedures \citep{suissa1991exact}.

\noindent\textbf{Statistical foundations.}
Beta--Binomial models describe clustered binary outcomes
\citep{kleinman1973,crowder1978,williams1975betabin,ridout1999icc} and motivate
\methodname{}'s within-group forecast. Its task selection is related to active
learning \citep{settles2009survey,houlsby2011bald} and adaptive replay
\citep{schaul2016per}, with non-degenerate group probability as the allocation
objective.
\section{Conclusion}
\label{sec:conclusion}

Reward disagreement provides a common lens for allocating evaluation and training
computation. \textsc{Discord} uses exact paired inference and decision-preserving
stopping to retain harness edits when observed disagreements support
improvement. \methodname{} uses capped historical
evidence, posterior uncertainty, and reward dependence to forecast mixed-reward
groups before generation. Across three model scales, \methodname{} achieves higher
observed held-out MCP solve rates than the compared methods,
reaches their strongest plotted peaks with fewer new rollouts, and transfers to
BFCL and NESTFUL without further
training. Together, the methods direct interaction budget toward harness
comparisons that can change retention and rollout groups that can supply a
group-relative training signal.

\label{pagecheck:mainend}
\clearpage

\subsection*{AI use statement}
\label{pagecheck:aiuse}

We used generative AI tools to assist with coding, including implementation,
debugging, and analysis scripts, and with reviewing relevant literature. The
authors directed these uses, checked the AI-assisted code and outputs, and
verified the resulting text and citations. All AI-assisted work was reviewed by
the authors, who take responsibility for the final content.

\section*{Ethics statement}
Improved tool-use efficiency can support beneficial applications but can also
increase the scale of unintended or unauthorized actions. Benchmark success and
lower harmful-edit retention do not establish deployment safety: ``harmful'' here
means reduced benchmark performance, not a comprehensive safety assessment.
Real-world use requires appropriate permissions, execution isolation, monitoring,
and safeguards beyond the reward verifier. Task exclusion and historical-data
reuse may also favor tasks already represented in the bank; our results do not
establish equitable performance across users or application domains.

\section*{Reproducibility statement}
Appendix~\ref{app:experimental} specifies model and training settings, task splits,
evaluation procedures, and checkpoint selection. Appendix~\ref{app:bank-provenance}
describes the source, caps, and initialization of historical evidence.
Allocation forecasts are recorded before the corresponding rollouts are generated;
paired tests use matched task indices. The code repository is available at \url{https://anonymous.4open.science/r/brace-38DC}.

\bibliographystyle{iclr2027_conference}
\bibliography{refs}

\clearpage
\appendix
This appendix discusses limitations and provides experimental protocols,
retention analyses, implementation details, and calibration controls supporting
the main results.

\section{Limitations}
\label{app:limitations}
\methodname{} predicts whether a rollout group has nonzero centered reward, not
whether its gradient improves the policy. The expected-yield interpretation in
Section~\ref{sec:brace-rho} therefore explains the allocation target but does
not establish that informative-group yield causes the observed solve-rate gains.
The bank uses rollouts generated by earlier policies and model scales. Although
no dedicated bank-collection run was launched here, those historical rollouts
have a cost separate from the current-run budget; cold-start or shifted task
pools may provide less useful prior evidence. The experiments use binary verifier
rewards, fixed-size groups,
and the tested MCP and transfer benchmarks; other rewards, group sizes, and tool
environments remain untested.

\textsc{Discord}'s exact test controls a prespecified paired comparison, not an
unrestricted sequence of adaptively proposed edits. Its conservative retention
can increase evaluation cost and miss useful edits under limited evidence, as
the live MCP comparison illustrates. It is therefore a safeguard for individual
retention decisions rather than an optimal harness-search policy.

\section{Experimental setup and evaluation}
\label{app:experimental}
\subsection{Models, training data, and optimization}
\label{app:trainconf}
\label{app:memberfull}
We train Qwen3-VL-Instruct models at $2$B, $4$B, and $8$B
\citep{bai2025qwen3vl}, using GRPO \citep{shao2024deepseekmath} in
\texttt{verl} \citep{sheng2025hybridflow}. The training substrate is Agent World
Model (AWM) \citep{wang2026agent}, whose MCP tools operate on SQLite-backed
state. A binary verifier evaluates the final database state. Training uses
$1{,}127$ tasks across $305$ live servers with a fixed tool surface. The held-out
AWM \citep{wang2026agent} set contains $295$ tasks disjoint from the training pool.

The bank snapshot used for the degeneracy diagnostic contains $482{,}510$ episodes on $1{,}493$ tasks from the same
substrate, including the training pool. Reconstructing the original optimizer steps
and retaining groups with at least two rollouts produces $92{,}706$ groups, of
which $85.2\%$ are degenerate. These bank-wide counts
are distinct from the task count and group composition of an individual run.
Only training-task outcomes initialize the allocator. All $295$ held-out AWM \citep{wang2026agent}
tasks are excluded from the initialization bank as well as from training.

\begin{table}[htbp]
\centering\small
\caption{Core training configuration. Allocation changes task selection before
rollout generation; evaluation uses a separate held-out task set.}
\label{tab:training-details}
\begin{tabularx}{\textwidth}{@{}>{\raggedright\arraybackslash}p{0.43\textwidth}>{\raggedright\arraybackslash}X@{}}
\toprule
Setting & Value \\
\midrule
Model family & Qwen3-VL-Instruct ($2$B, $4$B, $8$B) \\
Training framework and objective & \texttt{verl}, GRPO  \\
LoRA rank & $32$ \\
Prompts per optimizer step & $32$ \\
Rollouts per prompt & $5$ \\
Episodes per optimizer step & $160$ \\
Training pool & $1{,}127$ tasks, $305$ MCP servers \\
Held-out MCP evaluation & $295$ tasks, four evaluation seeds \\
Training replication & One training seed per configuration \\
\bottomrule
\end{tabularx}
\end{table}

\subsection{Baselines and comparison scope}
\label{app:parityfull}
The main comparison includes uniform GRPO \citep{shao2024deepseekmath},
Prioritized Level Replay (PLR) \citep{jiang2021plr}, RAG-MCP
\citep{gan2025ragmcp}, DAPO \citep{yu2025dapo}, Teacher--Student Curriculum
Learning (TSCL) \citep{matiisen2020tscl}, TRACE \citep{zou2026trace}, and VIP
\citep{nguyen2026vip}. These methods provide comparisons to uniform sampling,
replay and curriculum selection, tool retrieval, post-generation filtering, and
pre-generation allocation. The comparisons use the same model,
optimizer, current-run rollout budget, and evaluation protocol as specified in
Section~\ref{sec:triage-protocol}. Table~\ref{tab:main} compares final-checkpoint
performance, and Figure~\ref{fig:generalization} evaluates transfer. BRACE's
historical warm-bank initialization is described in
Appendix~\ref{app:main-allocation}; Table~\ref{tab:configs} evaluates its removal
alongside the other component comparisons.

The comparison unit matters for interpreting efficiency. Removing a degenerate
group after generation does not recover the episodes already executed. Likewise,
a warm bank avoids new probing episodes but requires previously collected data.
We distinguish the generation budget of the current run from historical data
acquisition and from the number of groups ultimately used for optimization.

\subsection{Statistical reporting}
\label{app:protofull}
For two policies evaluated on the same task--seed pairs, let $b$ count pairs won
only by the candidate and $c$ count pairs won only by the reference. The paired
solve-rate difference is $(b-c)/n$, where $n$ includes both concordant and
discordant pairs. Effects expressed in percentage points are $100(b-c)/n$;
a change from $11.4\%$ to $14.9\%$, for example, is $3.5$ pp.
The exact conditional McNemar test \citep{mcnemar1947} evaluates the discordant
count under $b\mid b+c\sim\mathrm{Binomial}(b+c,1/2)$, assuming independent pairs
and an exchangeable paired null. Repeated seeds on the same task need not form
independent pairs; the difference identity remains valid without this assumption,
but exact binomial error control does not.

Holm correction is applied within each registered comparison family. For $m$
ordered unadjusted values $p_{(1)}\le\cdots\le p_{(m)}$, the adjusted value at
rank $j$ is
\[
\widetilde p_{(j)}=\min\!\left\{1,\max_{\ell\le j}(m-\ell+1)p_{(\ell)}\right\}.
\]
Family membership is determined by the registered analysis, rather than by which
contrasts are favorable. Paired significance and absolute performance are reported separately;
the aggregation conventions are specified below.

\subsection{Evaluation and aggregation across runs}
\label{app:fivefull}
The training and evaluation seed counts are specified in Table~\ref{tab:training-details}.
Repeated evaluations characterize checkpoint-level stochasticity; they do not
estimate variability across independently trained policies. Main results use the
final plotted checkpoint, while learning curves retain the intermediate evaluations.

\subsection{Transfer benchmarks and processing}
\label{app:genfull}
BFCL v4 \citep{patil2025bfcl} evaluates multi-turn function calling on $800$ tasks
from four categories: \texttt{base}, \texttt{long\_context},
\texttt{miss\_func}, and \texttt{miss\_param}, with $200$ tasks each.
The official checker evaluates the executed result against the reference trajectory.
The evaluation uses a common task index across arms and one response per task
at temperature $0.001$. Transfer evaluation uses checkpoints from AWM \citep{wang2026agent} training,
without further training on BFCL \citep{patil2025bfcl}. The v4 multi-turn categories are separate from
the v3 single-turn split used for harness-retention experiments in
Appendix~\ref{app:loopfull}.

NESTFUL \citep{basu2024nestful} evaluates nested function composition. A completion
specifies API calls with dependencies, which are executed against the benchmark
implementations. The win-rate metric checks whether execution yields the recorded
answer; it differs from the benchmark's full-sequence matching metric.
Evaluation uses $1{,}861$ tasks, one completion
per task, and the same task index across arms. Two processing changes relative to
the released Llama branch are applied uniformly: function specifications are
embedded as JSON rather than as double-encoded strings, and completions are parsed
from the first balanced top-level \texttt{[...]} span. These changes preserve the
within-study comparison but prevent direct numerical comparison with results
obtained under the benchmark's original processing convention.

Figure~\ref{fig:generalization} reports percentage-point changes in BFCL \citep{patil2025bfcl} pass rate
and NESTFUL \citep{basu2024nestful} win rate relative to each scale's untrained base model, rather than
absolute rates. Neither benchmark supplies training data or task-selection feedback to the
reported AWM \citep{wang2026agent} training runs. Because the benchmarks use different task structures
and success criteria, their rates should be interpreted within each benchmark.

\section{Discord: derivations and retention experiments}
\label{app:discord}
\subsection{Inert-edit admission identity}
\label{app:identity}
Let $D=\bar B-\bar A$ under an inert edit. Exchangeability implies that $D$ and
$-D$ have the same distribution, so $P(D>0)=P(D<0)$. Partitioning the outcome
space into positive, negative, and zero differences yields
\[
1=P(D>0)+P(D<0)+P(D=0),\qquad
P(D>0)=\tfrac12[1-P(D=0)].
\]
This proves Equation~\ref{eq:chance}. The result concerns strict-improvement
retention; a different treatment of ties changes the admission probability.

We measure tie rates and inert-edit admission separately on AWM \citep{wang2026agent}, BFCL
\citep{patil2025bfcl}, and API-Bank \citep{li2023apibank}. The AWM \citep{wang2026agent} diagnostic uses
$6{,}000$ draws per caller ($12$ ordered interface pairs with $500$ draws each).
BFCL \citep{patil2025bfcl} and API-Bank \citep{li2023apibank} use $1{,}500$ group-of-four resampling trials per caller.
These estimators test the same identity at different sampling units; their rates
are not pooled. API-Bank's \citep{li2023apibank} published and $49$-tool surfaces are two settings of
one substrate, not independent benchmarks. Table~\ref{tab:identity-details}
provides the values underlying Figure~\ref{fig:identity}.

\begin{table}[htbp]
\centering
\caption{Independent predictions and measurements of inert-edit admission.
Difference is observed minus predicted admission probability.}
\label{tab:identity-details}
\begin{center}\small
\begin{tabular}{llrrrr}
\toprule
Substrate & Caller & Tie rate & Predicted & Observed & Difference \\
\midrule
AWM \citep{wang2026agent} & qwen    & 0.088 & 0.456 & 0.464 & $+0.008$ \\
    & granite & 0.082 & 0.459 & 0.455 & $-0.004$ \\
    & mistral & 0.149 & 0.426 & 0.427 & $+0.001$ \\
    & falcon3 & 0.107 & 0.447 & 0.446 & $-0.001$ \\
\cmidrule(l){2-6}
    & \emph{mean} & 0.106 & \textbf{0.447} & \textbf{0.448} & $+0.001$ \\
\midrule
BFCL \citep{patil2025bfcl} & qwen    & 0.223 & 0.389 & 0.377 & $-0.012$ \\
     & granite & 0.118 & 0.441 & 0.451 & $+0.010$ \\
     & mistral & 0.067 & 0.466 & 0.460 & $-0.006$ \\
     & falcon3 & 0.110 & 0.445 & 0.442 & $-0.003$ \\
\midrule
API-Bank \citep{li2023apibank} (published) & qwen    & 0.187 & 0.407 & 0.383 & $-0.024$ \\
                     & granite & 0.069 & 0.466 & 0.479 & $+0.013$ \\
                     & mistral & 0.058 & 0.471 & 0.477 & $+0.006$ \\
                     & falcon3 & 0.068 & 0.466 & 0.473 & $+0.007$ \\
\midrule
API-Bank \citep{li2023apibank} (49 tools)  & qwen    & 0.137 & 0.432 & 0.423 & $-0.009$ \\
                     & granite & 0.085 & 0.457 & 0.465 & $+0.008$ \\
                     & mistral & 0.051 & 0.474 & 0.459 & $-0.015$ \\
                     & falcon3 & 0.083 & 0.459 & 0.465 & $+0.006$ \\
\bottomrule
\end{tabular}
\end{center}
\end{table}

\subsection{Calibration and evaluation-budget planning}
\label{app:gatefull}
The admission identity is a property of exchangeability and the decision rule;
it does not establish false-retention control. In the inert-edit AWM \citep{wang2026agent} experiment,
strict-improvement retention accepts $43.3\%$ of
$1{,}440$ edits, and $43.6\%$ of retained edits fail to replicate.
For BFCL \citep{patil2025bfcl} and API-Bank \citep{li2023apibank}, holding the tool surface fixed and changing only random
seeds makes the comparison inert by construction.

Figure~\ref{fig:calib} examines a separate issue: the calibration and power of
candidate evaluation procedures. Both displayed $t$-test variants exceed the
nominal null false-positive rate for three of four callers, whereas the permutation
procedures remain calibrated. Pairing increases power under injected effects.
These observations motivate matched evaluation; the exact guarantee for
\textsc{Discord} follows from its conditional binomial test under the assumptions
in Section~\ref{sec:discord-test}.

\begin{figure}[htbp]
\centering
\includegraphics[width=\textwidth]{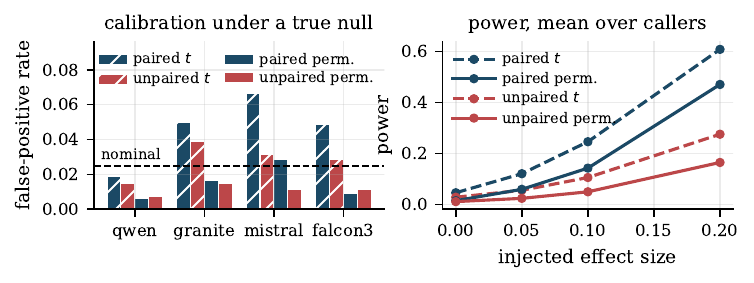}
\caption{Null calibration and injected-effect power across four callers. The
permutation procedures remain calibrated in the displayed diagnostic, and pairing
increases power. These comparisons motivate paired evaluation; \textsc{Discord}'s
error control follows from its specified exact conditional test.}
\label{fig:calib}
\end{figure}

For budget planning, let $p>1/2$ be the probability that a discordant pair favors
the candidate. Under the null, the standard deviation of the discordant win
fraction is $1/(2\sqrt{n_d})$; under the alternative, it is
$\sqrt{p(1-p)/n_d}$. Equating the approximate rejection threshold and the desired
power quantile gives
\[
(p-1/2)\sqrt{n_d}
\approx z_{1-\alpha}/2+z_{1-\beta}\sqrt{p(1-p)},
\]
and hence the planning rule
\begin{equation}
n_d\approx\frac{\bigl[z_{1-\alpha}/2+
 z_{1-\beta}\sqrt{p(1-p)}\bigr]^2}{(p-1/2)^2},
\qquad n_{\mathrm{tasks}}\approx n_d/\pi_d,
\label{eq:budget}
\end{equation}
where $\alpha$ is the one-sided test level, $1-\beta$ is target power, and $z_u$
is the standard-normal quantile at probability $u$. The expected number of tasks needed to
collect $n_d$ discordances is approximately $n_d/\pi_d$, where $\pi_d$ is the
discordance probability. Rare discordance therefore increases evaluation cost
even if the conditional effect is substantial. This normal approximation is used
for planning; retention uses the exact tail probability.

\subsection{Exact curtailment}
\label{app:seq}
Fix the maximum number of evaluation pairs and the terminal significance threshold
before inspecting outcomes. After each evaluation block, consider every feasible
completion of the remaining pairs: a pair can add a candidate win, an incumbent
win, or a concordance. Stop and accept only if every completion would be accepted
by the terminal test; stop and reject only if every completion would be rejected.
Otherwise, evaluate the next block.

For $R$ remaining pairs, this condition can be expressed using
\[
\mathcal C_R=\{(u,v):u,v\in\mathbb Z_{\ge0},\ u+v\le R\},
\]
where a completion yields counts $(b+u,c+v)$. Applying the terminal rule to all
members of $\mathcal C_R$ defines the two stopping decisions exactly.
When there are no discordances, the terminal rule rejects retention.
For $n_d=b+c>0$, the terminal one-sided tail is
\[
p_{\mathrm{tail}}(b,c)=2^{-n_d}\sum_{j=b}^{n_d}\binom{n_d}{j}.
\]
Retention requires $p_{\mathrm{tail}}\le\alpha$ for the prespecified level $\alpha$;
when $n_d=0$, set $p_{\mathrm{tail}}=1$ and reject retention.
The curtailed procedure \citep{alling1963curtailed} returns the same decision as
full evaluation on every completed sequence and thus inherits its error control.
It does not introduce an additional unadjusted test at each block. Savings depend
on effect size and the remaining budget; ambiguous comparisons can require the
full evaluation set.

\subsection{Controlled evolution loops}
\label{app:loopfull}
The controlled loops use the BFCL v3 \citep{patil2025bfcl} \texttt{live\_multiple} split,
restricted to $200$ single-turn tasks with binary
syntax-tree matching. Callers are Qwen3-VL-8B-Instruct
\citep{bai2025qwen3vl}, Granite-3.3-8B-Instruct \citep{ibm2025granite33},
Mistral-7B-Instruct-v0.3 \citep{mistralai2024mistral03}, and
Falcon3-7B-Instruct \citep{tii2024falcon3}, sampled at temperature $0.7$ with eight seeds per task
unless otherwise specified in the run records.

Each loop evaluates eight rounds of three proposed edits. Edits add or remove
blocks of advertised function names. Withholding names at initialization creates
a recoverable handicap; starting from the full surface instead produces a regime
in which proposals are predominantly harmful. The three retention rules consume
matched episode budgets and identical proposals. The untested rule retains a
strict empirical improvement; the single-round rule compares disjoint task halves
with a label-permutation test; \textsc{Discord} uses paired outcomes and its exact
conditional test.

\textsc{Discord} recovers more of the injected handicap in $18$ of $21$ paired
comparisons (Figure~\ref{fig:div}). The one-sided paired $t$-test gives
$p=0.00015$ ($t=4.37$, $20$ degrees of freedom); enumerating all $2^{21}$ sign
assignments gives an exact sign-flip value of $8.5\times10^{-5}$ under the
corresponding symmetry null. The single-round gate overstates realized effects
by $1.8$--$17\times$. In the harmful-proposal regime, \textsc{Discord} preserves
performance in all $17$ matched seeds, whereas strict-improvement retention
degrades in $9$ of $17$. The paired contingency table has nine seeds with
degradation only under strict-improvement retention, none with degradation only
under \textsc{Discord}, and eight with no degradation under either method.
The two-sided exact McNemar value is $2(1/2)^9=0.00390625$.
Mean changes are small, so this result concerns the
frequency of degradation rather than a large difference in mean improvement.

\subsection{\texorpdfstring{Live MCP comparison}{Live MCP comparison}}
\label{app:h2hsetup}
The live comparison uses $320$ AWM \citep{wang2026agent} tasks across $158$ servers,
with a shared proposal stream and a full-surface reference objective
of $0.2813$. The shallow regime withholds $10\%$ of tool names, producing a
recoverable effect of $2.2$ pp; the deep regime withholds $40\%$, producing
$12.7$ pp. Table~\ref{tab:headtohead} reports final objective values, acceptance
counts, and episode cost.

The deployed rule reaches a higher final objective with fewer evaluation episodes but
overstates it by $32\%$ and $47\%$ in the two regimes. \textsc{Discord} provides
a test-based retention decision at greater evaluation cost. This comparison separates
the reliability of reported improvement from the terminal objective achieved by a
search; the method does not dominate both criteria.

\begin{table}[htbp]
\centering\scriptsize
\setlength{\tabcolsep}{4pt}
\caption{Live AWM \citep{wang2026agent} retention comparison. Parenthesized final-objective percentages
are relative to the full-surface reference. Self-reported overstatement is relative
to the measured final objective; a dash indicates that no numerical self-report
is tabulated. Episode counts include evaluation cost.}
\label{tab:headtohead}
\vspace{4pt}
\begin{tabular}{@{}llcccr@{}}
\toprule
regime & rule & accepts (good/bad) & true final & self-reported & episodes \\
\midrule
shallow & deployed & 3 (2/1) & 0.2656 (94\%) & 0.3500 ($+32\%$) & 1,760 \\
shallow & \textsc{Discord} & 0 (0/0) & 0.2594 (92\%) & -- & 1,906 \\
\addlinespace[3pt]
deep & deployed & 3 (3/0) & 0.2812 (100\%) & 0.4125 ($+47\%$) & 1,920 \\
deep & \textsc{Discord} & 2 (2/0) & 0.2703 (96\%) & -- & 4,008 \\
\bottomrule
\end{tabular}

\end{table}

\section{BRACE: posterior computation and implementation}
\label{app:brace-details}
\subsection{Posterior evidence and sampling}
\label{app:main-allocation}
A task's success and failure evidence defines the working Beta approximation
$p_i\sim\mathrm{Beta}(a_i,b_i)$ with $a_i=1+s_i$ and $b_i=1+f_i$.
At initialization, the implementation caps historical evidence separately at
$20$ successful and $40$ failed episodes per task. The bank is incorporated once;
subsequent cycles add local outcomes and apply the same evidence-discount rule.
The registered configuration uses $\gamma=1$ and retains evidence without decay.
Setting $\gamma<1$ would reduce the influence of outcomes from older policies.

A cycle selects $256$ tasks, corresponding to eight optimizer steps of $32$
prompts each, with $k=5$ rollouts per prompt. Weighted sampling uses
$w_i^{1/\tau}$ with $\tau=0.5$ at all three scales. The exploration rate is $\epsilon=0.05$ at $2$B and $4$B,
and $0.10$ at $8$B. The implementation reserves
$n_\epsilon=\mathrm{round}(\epsilon N)$ of the $N=256$ selected rows for uniform
exploration and allocates the rest by weight. The uniform
exploration component samples from eligible tasks not already selected, as in
Algorithm~\ref{alg:brace}. Thus each selected task contributes one group in the
cycle. Task indices, selected groups, and predicted degeneracy are recorded before
the groups are generated.

A plug-in estimator substitutes a single estimated solve probability into the
availability score. Posterior integration averages over uncertainty in that
probability. Two tasks with the same posterior mean can therefore receive different
scores when their evidence counts differ. This is especially relevant early in training, before the
current run has collected substantial task-specific evidence.

\subsection{Bank provenance and cost accounting}
\label{app:bank-provenance}
The default bank reads other runs' episode logs once at initialization and excludes
the current arm's own log, which supplies its live evidence. Historical outcomes
come from policies at $2$B, $4$B, and $8$B. Thus initialization can transfer evidence
about task difficulty across model scales, including from larger to smaller models.
The cumulative audit inventory contains $195$ episode files: $107$ from $8$B,
$47$ from $2$B, and $41$ from $4$B policies. This inventory differs from an
individual run's initialization snapshot, which depends on the files available at
launch. Only eligible training-task outcomes contribute to initialization.

The VIP-score \citep{nguyen2026vip} and correlation-removal ablations use the same capped per-task starting
counts as full BRACE. Their bank-read totals can differ because the available files and redundant
observations differ, while capping produces the same retained evidence. Consequently,
more parsed episodes do not imply a larger effective prior in these comparisons.

The historical corpus and the capped evidence have different sizes. For example,
the recorded bank read for the $4$B arm \texttt{q4bT2} parses $651{,}804$ episodes
from $115$ files and primes $1{,}023$ tasks. Separate caps of $20$ successes and
$40$ failures bound its historical contribution by $61{,}380$ counts, or $9.4\%$
of the parsed episode count. This is an upper bound on retained historical counts;
it is neither the fraction of eligible episodes retained nor a compute-saving
estimate.

The bank reuses episodes generated by other experimental runs; no dedicated
collection run was launched to build it. This avoids additional probing for the
receiving run while retaining the historical collection cost and the benefit of
additional data. A current allocation cycle requests $256\times5=1{,}280$ episodes.
Historical generation, current-run generation, and the one-time bank read are
separate costs.

\subsection{Correlation correction and exact Beta moments}
\label{app:rhofull}
For a group on task $i$, draw a latent probability
$\theta_i\sim\mathrm{Beta}(\nu p_i,\nu(1-p_i))$ and draw its $k$ binary rewards
independently conditional on $\theta_i$. Marginalizing $\theta_i$ induces
within-group correlation $\rho=1/(\nu+1)$. Uncertainty about this group-level
probability is distinct from posterior uncertainty about the task mean $p_i$.

Write the unanimous-success probability as the degree-$k$ polynomial
\[
q(p)=\prod_{j=0}^{k-1}\frac{\nu p+j}{\nu+j}
=\sum_{m=0}^k c_m p^m.
\]
Using rising factorials $(x)^{\overline m}=\prod_{j=0}^{m-1}(x+j)$, with
$(x)^{\overline0}=1$, the score can be computed without Monte Carlo integration:
\[
w_i=1-\sum_{m=0}^k c_m
\frac{(a_i)^{\overline m}+(b_i)^{\overline m}}
     {(a_i+b_i)^{\overline m}}.
\]

\newpage
\subsection{Expected informative-group yield}
\label{app:yield-identity}
Fix the policy and the evidence available at allocation time. Let $I_i$ indicate
that a new group on task $i$ contains both reward values. By definition,
$w_i=\mathbb E[I_i]$ under the predictive model. If one equal-cost group is
generated for each task in a selected set $S$, linearity gives
\[
\mathbb E\left[\sum_{i\in S} I_i\mid S\right]=\sum_{i\in S}w_i.
\]
This holds without assuming independence across selected tasks. Under known
scores, selecting the $N$ largest $w_i$ maximizes the model's expected count of
informative groups among fixed-size sets. \methodname{} uses estimated scores,
weighted sampling, and uniform exploration instead of this deterministic oracle.
The identity interprets its allocation target; it does not establish that the
resulting gradient steps improve held-out solve rate.

\subsection{Eligibility and verifier checks}
\label{app:diagrest}
\label{app:triagefull}
\label{app:bridgefull}
Exclusion affects training allocation only. The $104$ excluded tasks comprise
$66$ initially satisfied verifiers, $33$ objectives ruled out by an initial-state
guard proof, and $5$ tasks identified by a string-matching heuristic. The final
category is not a proof-based certificate. The guard does not fire on any of the
$826$ bank tasks solved at least once. Evaluation and validation pools remain
unchanged.

A verifier that already accepts the initial state does not establish meaningful
task completion by a new trajectory. Conversely, repeated observed failures do not
prove that success is impossible. Structural checks and posterior evidence therefore
serve different roles: the checks determine eligibility, and the posterior ranks
and weights eligible tasks. The non-degeneracy objective in
Section~\ref{sec:brace-obj} motivates weighting both consistently successful and
consistently unsuccessful tasks below tasks likely to produce mixed rewards.

\section{Calibration, rank preservation, and allocation controls}
\label{app:calibration-controls}
\subsection{Prospective forecast evaluation}
\label{app:calib}
For each selected batch, the predicted degenerate fraction is the average of
$1-w_i$ over selected tasks. The observed fraction counts groups whose verifier
rewards are unanimous. Predictions are recorded before generation, using the
shared correlation estimate $\hat\rho=0.78$ specified in
Section~\ref{sec:brace-rho}. The $83$ evaluated arm-cycles exclude both cycles used to fit $\rho$;
the parameter is held fixed during this forecast evaluation. This separation is
distinct from excluding held-out AWM \citep{wang2026agent} tasks from bank initialization.
Mean absolute forecast error is $0.591$ under independence versus $0.056$ with
the correlated model.

Figure~\ref{fig:calibration} retains the diagnostic underlying this comparison.
It compares the two likelihood assumptions on the same cycles and the same
per-task estimates. Its interpretation is prospective batch calibration;
it does not by itself establish improvement in task ordering or downstream
policy quality. Those questions require the rank analysis and the held-out
comparisons, respectively.

\newpage
\begin{figure}[H]
\centering
\includegraphics[width=\textwidth]{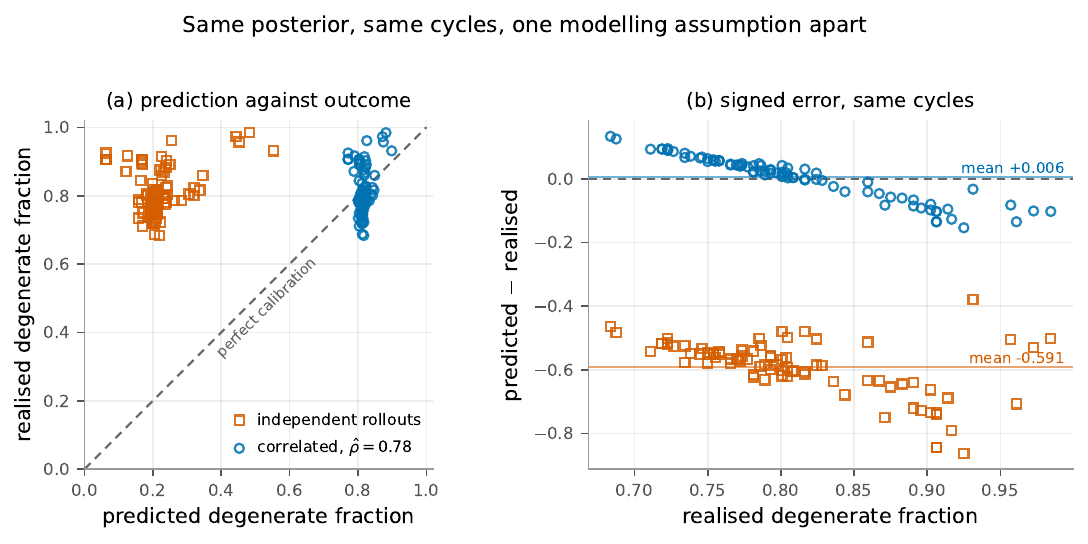}
\caption{Predicted and observed group degeneracy in the recorded $8$B calibration
diagnostic. Left: predicted versus observed fractions. Right: signed forecast
errors. Mean absolute errors are $0.591$ and $0.056$ for the independent and
correlated models, respectively; the signed mean errors annotated in the plot
are $-0.591$ and $+0.006$. The diagnostic contains $83$ arm-cycles from $21$ arms.}
\label{fig:calibration}
\end{figure}

\subsection{Proof of rank preservation}
\label{app:proofs}
For the independent score, differentiation gives
\[
\frac{\partial}{\partial p}[p^k+(1-p)^k]
=k[p^{k-1}-(1-p)^{k-1}]<0,\qquad 0<p<1/2.
\]
Both availability and $u=p(1-p)$ are symmetric about $p=1/2$, and both increase
strictly on $[0,1/2]$. Thus independent availability is strictly increasing in $u$.
For the correlated case, use the function $f$ defined in
Proposition~\ref{prop:rankcal}.
\begin{proof}[Proof of Proposition~\ref{prop:rankcal}]
$f$ is a product of positive increasing linear functions of $p$, so as a polynomial
$f(p)=\sum_m c_m p^m$ it has $c_m\ge0$ with $c_0=0$ and $c_k=\nu^k/\prod_{j<k}(\nu+j)>0$. Then
$f''(p)=\sum_{m\ge2}m(m-1)c_mp^{m-2}>0$ on $(0,1]$ because $c_k>0$ and $k\ge2$, so $f'$ is strictly
increasing and $g_c'(p)=f'(1-p)-f'(p)>0$ for $p<\tfrac12$. As $g_c(p)=g_c(1-p)$ by construction and
$u\mapsto p=(1-\sqrt{1-4u})/2$ is a strictly increasing bijection $[0,\tfrac14]\to[0,\tfrac12]$, the
composition is strictly increasing. At the endpoint $p=\tfrac12$ the derivative in $u$ is $0/0$ with
limit $f''(\tfrac12)$, which is positive for all $\nu>0$,
$k\ge2$ by the same coefficient argument.
\end{proof}

\end{document}